\documentclass[runningheads]{llncs}

\usepackage[utf8]{inputenc}
\usepackage[english]{babel}
\usepackage{amsmath,amsbsy,bm,amssymb,mathtools,dsfont,etoolbox}
\usepackage{xspace,geometry,graphicx,float}
\usepackage[hidelinks]{hyperref}
\usepackage{cleveref}
\usepackage{enumitem}
\usepackage{siunitx}
\usepackage{caption,subcaption}
\usepackage{booktabs}
\usepackage{scalerel}

\crefname{figure}{Fig.}{Figs.}
\Crefname{figure}{Fig.}{Figs.}
\crefname{table}{Tab.}{Tabs.}
\Crefname{table}{Tab.}{Tabs.}
\crefname{equation}{}{}
\Crefname{equation}{}{}
\crefname{section}{Sec.}{Secs.}
\Crefname{section}{Sec.}{Secs.}
\crefname{subsection}{Sec.}{Secs.}
\Crefname{subsection}{Sec.}{Secs.}
\crefname{definition}{definition}{definitions}
\Crefname{definition}{Definition}{Definitions}

\newcommand{\ie}{i.\,e.\@\xspace}
\newcommand{\eg}{e.\,g.\@\xspace}
\newcommand{\cf}{cf.\@\xspace}
\newcommand{\iid}{i.i.d.\@\xspace}
\newcommand{\wrt}{w.r.t.\@\xspace}

\newcommand{\ind}[1]{\mathbb{I}\lbrace #1 \rbrace}
\newcommand{\BP}{\mathbb{P}}
\newcommand{\BQ}{\mathbb{Q}}
\newcommand{\BH}{\mathbb{H}}
\newcommand{\BG}{\mathbb{G}}
\newcommand{\proofqed}{\hfill$\Box$}

\renewcommand{\orcidID}[1]{}

\begin{document}

\title{Shapley Values with Uncertain Value Functions}
\titlerunning{Uncertain Shapley Values}
\author{Raoul Heese\inst{1}\orcidID{0000-0001-7479-3339} \and
	Sascha Mücke\inst{2}\orcidID{0000-0001-8332-6169} \and
    Matthias Jakobs\inst{2}\orcidID{0000-0003-4607-8957} \and
	Thore Gerlach\inst{3}\orcidID{0000-0001-7726-1848} \and
	Nico Piatkowski\inst{3}\orcidID{0000-0002-6334-8042}}
\authorrunning{R. Heese et al.}
\institute{Fraunhofer ITWM \email{raoul.heese@itwm.fraunhofer.de} \and TU Dortmund \and Fraunhofer IAIS}
\maketitle

\begin{abstract}
We propose a novel definition of Shapley values with uncertain value functions based on first principles using probability theory. Such uncertain value functions can arise in the context of explainable machine learning as a result of non-deterministic algorithms. We show that random effects can in fact be absorbed into a Shapley value with a noiseless but shifted value function. Hence, Shapley values with uncertain value functions can be used in analogy to regular Shapley values. However, their reliable evaluation typically requires more computational effort.
\end{abstract}

\section{Introduction}
The ability to interpret the predictions of machine learning (ML) models is an important requirement for many data-driven decision support applications, \eg, in computational biology \cite{watson2022}, medicine \cite{amann2020}, materials science \cite{zhong2022}, and finance \cite{salvatore2022}, just to name a few. There are a variety of model-agnostic methods that enable explainability of otherwise black-box models \cite{molnar2022}. The theory of Shapley values (SVs) \cite{shapley1953}, a concept from coalitional game theory, builds the foundation for a collection of such methods \cite{gromping2007,strumbelj2010,strumbelj2014,lundberg2017,merrick2019,sundararajan2020}. They all have in common that a SV---in the sense of an importance score---is attributed to each feature of an arbitrary predictive model. Based on these scores, explanations can be made about which features are responsible for which prediction. For a recent review of SVs in explainable ML (XML), see, \eg, \cite{belle2021,rozemberczki2022} and references therein.\par
SVs for XML are based on the premise of a value function (VF) that quantifies the impact of each feature to a model's prediction by a real number. There are various approaches to define suitable VFs for such a task \cite{sundararajan2020}, but their evaluation typically involves a statistical or randomized (\ie, non-deterministic) ingredient. For example, a straightforward approach for a classifier is to retrain the model for all possible feature combinations and use its prediction of class probabilities to specify the VF \cite{strumbelj2009}. However, the training procedure os often a randomized algorithm, \eg, due to random initializations, data splitting, or randomized search heuristics---leading to different outcomes if repeated with different random seeds. As a result, the corresponding VF becomes a random variable that yields different outcomes if evaluated multiple times.\par
The effect of uncertain VFs for SVs in the context of XML has not been extensively studied in the literature before. To our knowledge, \cite{li2020}, which is based on the \emph{DeepSHAP} approximation \cite{shrikumar2017,lundberg2017}, is the only work in this direction. Therein, a purely empirical perspective on sampling uncertainty is studied without any explicit sources of randomness. More generally, in \cite{kargin2005,fatima2006}, the authors define the uncertainty of coalition games as the standard deviation when considering the marginal contributions of each player entering a random coalition, which can also be understood as a measure of the strategic risk. An important insight is that whenever two players have the same SVs, they do not necessarily share the same risk. However, no randomness from the VF is considered in these works. In \cite{gao2017,dai2022}, the authors study uncertain coalition games \cite{yang2013} based on uncertainty theory \cite{liu2010} and introduce risk averse SVs. A brief review of related work can also be found in \cite{dai2022}. Since the calculus of uncertainty theory is based on special assumptions, the results are not straightforwardly applicable to the uncertainties that arise for SVs in the context of XML. In conclusion, there is
no self-consistent work on SVs with uncertain VFs based on probability theory.\par
We think, however, that the informative value of explainability via SVs can only be fairly judged if it incorporates a sound discussion of such uncertainties. As stated in \cite{hart1989}, ``it should be emphasized that the value of a game is an a priori measure---before the game is actually played.'' And indeed, SVs can only provide an expectation value of the feature importance. Especially in this light, it seems almost natural to ask the question what influence uncertainty has on the underlying probability density of this expected value and what happens to the higher moments. In this manuscript, we try to get a little closer to answering such questions.\par
Specifically, we present two major contributions:
\begin{enumerate}
	\item In a novel approach, we define SVs with uncertain VFs based on first principles using probability theory.
	\item We show that SVs with uncertain VFs correspond to regular SVs with shifted (deterministic) VFs, where the shift is determined by the mean bias of the marginal contributions. Our proposed definition consequently fulfills all properties of regular SVs. Furthermore, SVs with uncertain VFs can be used in analogy to regular SVs but with a higher computational effort that depends on the desired confidence.
\end{enumerate}
The remaining part of this manuscript is structured as follows. In \cref{sec:shapley values}, we review regular SVs (without uncertainty) from a mostly general perspective. Subsequently, we introduce our definition of SVs with uncertain VFs in \cref{sec:uncertain shapley values} and explain their properties. In \cref{sec:xml}, we briefly discuss the implications for XML and show a numerical experiment. Finally, we conclude with a brief summary and outlook in \cref{sec:summary}.
Throughout this paper, uncertain quantities are denoted in boldface. Moreover, lowercase letters denote a probability density that corresponds to its respective probability measure (\eg, $p$ and $\BP$).

\section{Shapley Values} \label{sec:shapley values}
SVs, which originated in game theory \cite{shapley1953,aumann1974}, are based on the following premise: A coalition (or group) of players cooperates in what is called a coalition game and achieves a certain profit from this cooperation. The question that SVs try to answer is to evaluate the individual contribution of each player to the overall outcome. In other words, SVs represent an approach to attribute the total gains of a coalition game to each individual of a group of participating players. In this section, we start with a formal definition of SVs and then discuss some of their properties.

\subsection{Definition}
A coalition game is in this context defined by the VF (also called characteristic function), which specifies the payoff of the game as a real number under the presumption that only a subset of players participate:
\begin{definition}[\cf~\cite{shapley1953}]
    Let $S \subseteq \mathcal{S} := \{ 1, \dots, N \}$ be a coalition of $N$ players. A function 
    \begin{equation} \label{eqn:v}
    	v : S \rightarrow \mathbb{R}
    \end{equation}
	that assigns coalitions to profit is called a value function (VF).
\end{definition}
In total, there are $2^N$ possible coalitions including the empty set.\par
The individual contribution of each player to the game is represented by the associated SV:
\begin{definition}[\cf~\cite{shapley1953}]
    Given a player $i \in \{ 1, \dots, N \}$ and a VF $v$, the corresponding Shapley value (SV)
    \begin{equation} \label{eqn:phi}
        \Phi_i(v) \equiv \Phi_i := \sum_{S \subseteq \mathcal{S} \setminus \{i\}} w(S) \Delta_i v (S) 
    \end{equation}
	is a weighted sum over its marginal contributions
	\begin{equation} \label{eqn:deltav}
		\Delta_i v (S) := v(S \cup \{i\}) - v(S)
	\end{equation}
	when added to all coalitions of other players. The corresponding weights
	\begin{equation} \label{eqn:w}
		w(S) \equiv w(\vert S \vert) := \frac{1}{N} {N-1 \choose \vert S \vert}^{-1} = \frac{\vert S \vert ! \left( N  - \vert S \vert - 1 \right) !}{N !}
	\end{equation}
	depend on the cardinality of each coalition. 
\end{definition}
By default, the VF is typically normalized such that $v(\varnothing) := 0$, where $\varnothing$ denotes the empty set.
This normalization can be achieved by a linear shift $v(S) \mapsto v(S) - v(\varnothing)$ in \cref{eqn:v}.
However, since the VF enters twice in \cref{eqn:phi} in form of a subtraction, such a shift has no effect on $\Phi_i$.
Hence, we do not presume this kind of normalization in the following.

\subsection{Properties}
SVs exhibit many desirable properties, for example:
\begin{description}[font=\bfseries, leftmargin=25mm, style=nextline]
	\item[Efficiency]
	$\sum_{i=1}^N \Phi_i = v(\mathcal{S}) - v(\varnothing)$.
	\item[Symmetry]
	If $\forall\, S \subseteq \mathcal{S} \setminus \{i,j\}$, $v(S \cup \{i\}) = v(S \cup \{j\})$, then $\Phi_i=\Phi_j$.
	\item[Linearity]
	For two VFs $v$ and $v'$, $\Phi_i(v) + \Phi_i(v') = \Phi_i(v + v') \,\forall\, i \in \mathcal{S}$. For $\alpha \in \mathbb{R}$, $\alpha \Phi_i(v) = \Phi_i(\alpha v) \,\forall\, i \in \mathcal{S}$.
	\item[Null Player]
	If $\forall\, S \subseteq \mathcal{S} \setminus \{i\}$, $v(S \cup \{i\}) = v(S)$, then $\Phi_i = 0$.
\end{description}
In fact, it can be shown that SVs represent the unique attribution method that satisfies these four properties simultaneously \cite{dubey1975}.

\subsection{Probabilistic View}
The sum of the weights from \cref{eqn:w} over all possible coalitions is normalized according to $\sum_{S \subseteq \mathcal{S} \setminus \{i\}} w(S) = 1$, which allows us to view the coalitions in \cref{eqn:phi} as a random variable $\bm{S} \sim \BP$ over the set $\mathcal{S}\setminus\lbrace i\rbrace$ with probability mass function 
\begin{equation} \label{eqn:Pw}
	\BP(S) := \BP(\bm{S}=S) = w(S).
\end{equation}
This approach, in turn, also makes the marginal contributions from \cref{eqn:deltav} a random variable, that is, $\bm{V}_i := \Delta_i v(\bm{S}) \sim \BP_i$ with probability mass function
\begin{equation} \label{eqn:deltav:p}
	\BP_i(u) := \BP_i(\bm{V}_i=u) = \sum_{S \subseteq \mathcal{S} \setminus \{i\}} \BP(S) \cdot\ind{\Delta_i v(S)=u},
\end{equation}
where the indicator function $\ind{\Delta_i v(S)=u} \in \{0,1\}$ is $1$ if $\Delta_i v(S)=u$ and $0$ otherwise. By construction, $\bm{V}_i$ has only finite support on the discrete set $\Delta_i\mathcal{V} := \lbrace \Delta_i v(S) ~|~ S\subseteq\mathcal{S}\setminus\lbrace i\rbrace\rbrace$. 
These random variables allow a different perspective on \cref{eqn:phi} in form of expectation values. Specifically,
\begin{equation} \label{eqn:phi:E}
	\Phi_i = \mathbb{E} \left[\bm{V}_i \right] = \mathbb{E} \left[ \Delta_i v(\bm{S}) \right].
\end{equation}
Hence, the determination of $\Phi_i$ can also be understood as estimating one of the corresponding expectation values.

\subsection{Higher Moments}
Since we have full knowledge about the probability mass function $\BP_i$, we can also evaluate higher moments. Specifically, the $n$th moment reads
\begin{equation} \label{eqn:phi:moments}
	\mathbb{E} \left[ \bm{V}_i^n \right] = \sum_{S \subseteq \mathcal{S} \setminus \{i\}} w(S) \left[ \Delta_i v(S) \right]^n = \mathbb{E} \left[ \Delta_i v(\bm{S})^n \right]
\end{equation}
for $n \in \mathbb{N}$. The first and second moments can be used to calculate the variance \cite{kargin2005,fatima2006}
\begin{equation} \label{eqn:phi:V}
	\sigma_{i}^2 := \mathbb{E} \left[ \bm{V}_i^2 \right] - \mathbb{E} \left[ \bm{V}_i \right]^2 = \mathbb{E} \left[ \Delta_i v(\bm{S})^2 \right] - \Phi_i^2.
\end{equation}
The variance (or the corresponding standard deviation $\sigma_{i}$) is a measure for the dispersion of the marginal contributions.\par 
Therefore, we consider \cref{eqn:phi:V} as a measure of the \emph{intrinsic} uncertainty of $\Phi_i$. With ``intrinsic'' we refer to the fact that the VF $v(S)$ is (so far) presumed to be deterministic such that an uncertainty of $\Phi_i$ can only arise from the distribution of marginal contributions $\bm{V}_i$ that is determined by the choice of $v(S)$. In the next section, we discuss uncertain VFs and their implication for SVs.

\section{Uncertain Shapley Values} \label{sec:uncertain shapley values}
The evaluation of a VF that is determined by a randomized procedure introduces uncertainty. In typical ML pipelines, randomization may occur as part of the learning algorithm (\eg, dropout \cite{srivastava2014}), and by the random choice of the training data. A plethora of methods for quantifying and mitigating uncertainty has been studied extensively in ML under the umbrella of Bayesian methods. Surprisingly, such methods are not well-studied in the context of SVs, as described in the introduction.\par
In this section, we present our definition of SVs with uncertain VFs, or \emph{uncertain SVs} for short. First, we introduce uncertain VFs as random variables that lead to a probability distribution with an expectation value corresponding to the uncertain SVs. Next, we discuss various properties that arise from this definition.

\subsection{Uncertain Value Functions}
We presume that the VF $v (S)$ in \cref{eqn:phi} is replaced by an uncertain VF in form of a random variable:
\begin{definition}
	Let $S \subseteq \mathcal{S} := \{ 1, \dots, N \}$ be a coalition of $N$ players. The random variable
	\begin{equation} \label{eqn:v:uncertainty}
		{\boldsymbol v} (S) \equiv \boldsymbol v := v (S) + \boldsymbol\nu(S)
	\end{equation}
	that consists of the sum of a VF $v (S)$ as defined in \cref{eqn:phi} and a random variable $\boldsymbol\nu(S) \equiv \boldsymbol\nu \sim \BQ(\cdot~|~S)$ that represents uncertainty or noise in the determination of $v(S)$ is called an uncertain VF. The associated (conditional) probability measure $\BQ(\cdot~|~S)$ has real support but is not constrained by additional assumptions.
\end{definition}
Our definition is very generic and covers both aleatoric uncertainty (from stochastic errors) and epistemic uncertainy (from systematic errors).
Based on this new kind of VF, the marginal contributions $\Delta_i v (S)$, \cref{eqn:phi}, can be replaced according to $\Delta_i v (S) \rightarrow \Delta_i {\boldsymbol v} (S)$ by the marginal contribution under uncertainty $\Delta_i {\boldsymbol v} (S) := \Delta_i v (S) + \boldsymbol\epsilon_i (S)$, where $\boldsymbol\epsilon_i (S) \equiv \boldsymbol\epsilon_i := \boldsymbol\nu(S \cup \{i\}) - \boldsymbol\nu(S) \sim \BH_i (\cdot~|~S)$ denotes a random variable with density 
\begin{equation} \label{eqn:deltav:q} 
	h_i (\epsilon~|~S) := \int_{\mathbb{R}} q(\nu~|~S \cup \{i\}) \ q(\nu - \epsilon~|~S) \, \mathrm d \nu,
\end{equation}
which corresponds to a convolution.

\subsection{Probabilistic View}
With these presumptions, we arrive at the random variable
\begin{equation} \label{eqn:V:uncertainty}
	\tilde{\bm{V}}_i:=\bm{V}_i+\bm \epsilon_i  = \Delta v_i (\bm{S}) + \boldsymbol\nu(\bm{S} \cup \{i\}) - \boldsymbol\nu(\bm{S})
\end{equation}
with $\bm{S}_i \sim \BP$ and $\boldsymbol\nu(S) \equiv \boldsymbol\nu \sim \BQ(\cdot~|~S)$, respectively. In particular, the discrete measure $\BP_i$, \cref{eqn:deltav:p}, becomes continuous such that $\tilde{\bm{V}}_i \sim \tilde{\BP}_i$ with
\begin{equation} \label{eqn:deltav:p:uncertainty} 
	\tilde{p}_i(u) = \sum_{S \subseteq \mathcal{S} \setminus \{i\}} \BP(S) \ h_i (u - \Delta_i v(S)~|~S) .
\end{equation}
This expression can also be rewritten as $\tilde{p}_i(u) = \sum_{S \subseteq \mathcal{S} \setminus \{i\}} w(S) \int_{\mathbb{R}} h_i(\epsilon~|~S) \delta(\Delta_i v(S)+\epsilon-u) \, \mathrm d \epsilon$ to highlight the similarity between $\tilde{p}_i$ and its noiseless analogon $\BP_i$ from \cref{eqn:deltav:p}. Specifically, the indicator function in \cref{eqn:deltav:p} translates to the Dirac delta function $\delta$ under an integral over the probability density $h_i$.\par
These preliminary considerations allow us to define uncertain SVs in analogy to \cref{eqn:phi:E}:
\begin{subequations} \label{eqn:phi:E:uncertainty}
\begin{definition}
	Given a player $i \in \{ 1, \dots, N \}$ and an uncertain VF ${\boldsymbol v} (S)$ as defined in \cref{eqn:v:uncertainty}, the corresponding uncertain SV
	\begin{equation} \label{eqn:phi:E:uncertainty:VS}
		\tilde{\Phi}_i(\boldsymbol v) \equiv \tilde{\Phi}_i := \mathbb{E} \left[ \tilde{\bm{V}}_i \right]
	\end{equation}
	is the expectation value of a random variable $\tilde{\bm{V}}_i$ as given by \cref{eqn:V:uncertainty,eqn:deltav:p:uncertainty}.
\end{definition}
The expression \cref{eqn:phi:E:uncertainty:VS} can be rewritten as
\begin{equation} \label{eqn:phi:E:uncertainty:V}
	\tilde{\Phi}_i = \sum_{S \subseteq \mathcal{S} \setminus \{i\}} w(S) \Delta_i^{\epsilon} \bm{v} (S) \quad \text{with} \quad \Delta_i^{\epsilon} \bm{v} (S) := \Delta_i v (S) + \mathbb{E} \left[ \bm{\epsilon}_i~|~S\right]
\end{equation}
by exploiting the linearity of the expectation value. These two forms of $\tilde{\Phi}_i$ highlight different perspectives, similar to \cref{eqn:phi,eqn:phi:E}. In \cref{eqn:phi:E:uncertainty:VS}, two random variables $\bm{V}_i$ (from the sum over coalitions) and $\bm \epsilon_i$ (from uncertain VFs) contribute to the expectation value. On the other hand, in \cref{eqn:phi:E:uncertainty:V} only one random variable $\bm \epsilon_i$ is considered. This form can also be found in \cite{gao2017}. A third representation
\begin{equation} \label{eqn:phi:E:uncertainty:G}
	\tilde{\Phi}_i = \Phi_i + \Gamma_i \quad \text{with} \quad \Gamma_i(\boldsymbol v) \equiv \Gamma_i(\boldsymbol \nu) \equiv \Gamma_i := \sum_{S \subseteq \mathcal{S} \setminus \{i\}} \BP(S) \mathbb{E} \left[ \bm{\epsilon}_i~|~S\right] = \mathbb{E} \left[ \bm{\epsilon}_i (\bm{S}) \right]
\end{equation}
\end{subequations}
can be obtained from \cref{eqn:phi:E:uncertainty:V} if we recall $\Phi_i$ from \cref{eqn:phi}. This representation particularly highlights that the difference between uncertain SVs and regular (\ie, noiseless) SVs is given by $\Gamma_i$ as an expectation value of $\bm{\epsilon}_i (\bm{S}) \sim \BG_i$ with ${g}_i(\epsilon) = \sum_{S \subseteq \mathcal{S} \setminus \{i\}} \BP(S) \ h_i (\epsilon ~|~S)$.\par
By definition, the mean of the marginal contribution noise $\mathbb{E} \left[ \bm{\epsilon}_i~|~S\right] = \mathbb{E} \left[ \boldsymbol\nu(S \cup \{i\}) ~|~S\right] - \mathbb{E} \left[ \boldsymbol\nu(S)~|~S\right]$ covers the noise-induced expectation value shift when adding a player $i$ to the coalition $S$. Hence, $\Gamma_i$ can be considered as the mean bias of the marginal contributions of player $i$. If the expectation value of the noise is independent of $S$ (\eg, because $h_i$ is independent of $S$), $\Gamma_i$ vanishes.\par
Summarized, we have found three representations with uncertain SVs, \cref{eqn:phi:E:uncertainty:VS,eqn:phi:E:uncertainty:V,eqn:phi:E:uncertainty:G}. All of them are identical, but correspond to different interpretations as we have explained. In the borderline case of vanishing uncertainty (\ie, $\bm \nu(S) \equiv 0$), the VF becomes deterministic again since $q(\nu\ |\ S) = \delta(\nu)$, which leads to $h_i (\epsilon\ |\ S) = \delta(\epsilon)$ and therefore $\Gamma_i = 0$. Accordingly, the uncertain SV reduces to the regular SV, as expected.

\subsection{Properties}
\label{sec:uncertain_properties}
Our previous considerations allow us to establish a direct connection between uncertain SVs and regular SVs:
\begin{theorem} \label{theorem1}
	Let $\mathbf v(S) = v(S) + \boldsymbol \nu(S)$ be an uncertain VF in the sense of \cref{eqn:v:uncertainty} with a deterministic part $v(S)$ and a random part $\boldsymbol \nu(S)$. We denote the corresponding uncertain SV, \cref{eqn:phi:E:uncertainty}, by $\tilde{\Phi}_i(\bm v)$. Then, $\tilde{\Phi}_i(\bm v) = \Phi_i(v')$ for all $i \in \mathcal{S}$, where $\Phi_i(v')$ denotes the regular SV, \cref{eqn:phi}, with respect to a deterministic VF $v'(S) := v(S) + \gamma(S)$, where $\gamma(S) := \sum_{j \in S} \Gamma_j$ is based on $\Gamma_j$ from \cref{eqn:phi:E:uncertainty:G}.
\end{theorem}
\begin{proof}
	First, take note that for all $\Gamma_i \in \mathbb{R}$ and $i \in \mathcal{S}$, we can define the noiseless VF $v''(S) := \gamma(S) = \sum_{j \in S} \Gamma_j$ for all $S \subseteq \mathcal{S}$ with the effect that $\Phi_i(v'') = \Gamma_i$ for all $i \in \mathcal{S}$. Second, as a consequence of \cref{eqn:phi:E:uncertainty:G} and the Linearity property of regular SVs, $\tilde{\Phi}_i(\bm v) = \Phi_i(v) + \Gamma_i = \Phi_i(v) + \Phi_i(v'') = \Phi_i(v + v'')$ for all $i \in \mathcal{S}$. Finally, define $v'(S) := v(S) + v''(S)$ for all $S \subseteq \mathcal{S}$.\proofqed
\end{proof}
Hence, uncertain SVs can in fact be considered as regular SVs with a suitably shifted VF. The shift $\gamma(S)$ corresponds to the cumulated mean bias of the marginal contributions for all players in $S$.\par
As a consequence, the properties of regular SVs (Efficiency, Symmetry, Linearity, Null Player) are also uniquely fulfilled with uncertain SVs. Moreover, using \cref{eqn:phi:E:uncertainty:G}, we can straightforwardly rewrite these properties in a form that highlights the uncertainty representation:
\begin{description}[font=\bfseries, leftmargin=42mm, style=nextline]
	\item[Uncertain Efficiency]
	$\sum_{i=1}^N \tilde{\Phi}_i = v(\mathcal{S}) - v(\varnothing) + \gamma(\mathcal{S})$.
	\item[Uncertain Symmetry]
	If $\forall\, S \subseteq \mathcal{S} \setminus \{i,j\}$, $v(S \cup \{i\}) = v(S \cup \{j\})$, then $\tilde{\Phi}_i=\tilde{\Phi}_j + \Gamma_i - \Gamma_j$.
	\item[Uncertain Linearity]
	For two noisy VFs $\bm v$ and $\bm v'$, $\tilde{\Phi}_i(\bm v) + \tilde{\Phi}_i(\bm v') = \tilde{\Phi}_i(\bm v + \bm v') \,\forall\, i \in \mathcal{S}$. For $\alpha \in \mathbb{R}$, $\alpha \tilde{\Phi}_i(\bm v) = \tilde{\Phi}_i(\alpha \bm v) \,\forall\, i \in \mathcal{S}$.
	\item[Uncertain Null Player]
	If $\forall\, S \subseteq \mathcal{S} \setminus \{i\}$, $v(S \cup \{i\}) = v(S)$, then $\tilde{\Phi}_i = \Gamma_i$.
\end{description}
Due to the linearity of the expectation value, $\Gamma_i(\bm v + \bm v') = \Gamma_i(\bm v) + \Gamma_i(\bm v')$ and $\Gamma_i(\alpha \bm v) = \alpha \Gamma_i(\bm v)$, respectively. We have also made use of $\gamma(\mathcal{S}) = \sum_{i=1}^N \Gamma_i$. In addition to the properties of regular SVs, the properties of uncertain SVs from \cite{gao2017} apply straightforwardly based on \cref{eqn:v:uncertainty,eqn:phi:E:uncertainty:V}. We refer to \cite{gao2017} for a derivation. 

\subsection{Higher Moments}
Higher moments of the random variable $\tilde{\bm{V}}_i$ can be determined similar to the noiseless case, \cref{eqn:phi:moments}. Specifically, the $n$th moment reads
\begin{equation} \label{eqn:phi:moments:uncetainty}
	\mathbb{E} \left[\tilde{\bm{V}}_i^n \right] = \sum_{k=0}^{n} \binom{n}{k} \sum_{S \subseteq \mathcal{S} \setminus \{i\}} w(S) \left[ \Delta_i v (S) \right]^k \mathbb{E} \left[ \bm{\epsilon}_i^{n-k} \ |\ S\right]
\end{equation}
for $n \in \mathbb{N}$, where $\mathbb{E} \left[ \bm{\epsilon}_i^n\ |\ S\right] = \sum_{k=0}^{n} \binom{n}{k} (-1)^{n-k} \mathbb{E} \left[ \boldsymbol\nu^{k}\ |\ S \cup \{i\} \right] \mathbb{E} \left[ \boldsymbol\nu^{n-k}\ |\ S \right]$ follows from \cref{eqn:deltav:q}. Both equalities are direct implications of the binomial theorem. Consequently, any moment of $\tilde{\bm{V}}_i$ can be explicitly calculated via \cref{eqn:phi:moments:uncetainty} based on the corresponding moments of the random variable $\boldsymbol\nu$ --- knowledge about the underlying probability measures is not required.\par
For example, the variance reads
\begin{equation} \label{eqn:phi:V:uncertainty}
	\tilde{\sigma}_{i}^2:= \mathbb{E} \left[\tilde{\bm{V}}_i^2 \right] - \mathbb{E} \left[\tilde{\bm{V}}_i \right]^2  =  \sigma_{i}^2 + \sigma_{\Gamma_i}^2 + \xi_i
\end{equation}
with $ \sigma_{i}^2$ from \cref{eqn:phi:V}, the variance of the noise of the marginal contributions $\sigma_{\Gamma_i}^2 := \mathbb{E} \left[ \bm{\epsilon}_i^2 \right] - \Gamma_i^2$, and the correlation $\xi_i := 2 \sum_{S \subseteq \mathcal{S} \setminus \{i\}} w(S) \Delta_i v (S) \mathbb{E} \left[ \bm{\epsilon}_i\ |\ S \right] - 2 \Phi_i \Gamma_i$.
Hence, we find that the noise from the VF introduces a noise uncertainty $\sigma_{\Gamma_i}^2 \geq 0$ and a correlation term $\xi_i \in \mathbb{R}$ in addition to the intrinsic uncertainty $\sigma_{i}^2 \geq 0$.

\section{Explainable Machine Learning} \label{sec:xml}
When the concept of SVs is used for XML, features (or feature indices, to be more precise) take the role of players and the VF is determined by the model output. Suitable VFs can be realized in many different variants leading to different kinds of SVs \cite{sundararajan2020}. For example, explanations can be performed for a single data point or an entire data set and might or might not require a retraining of the model. Uncertain SVs can be used in analogy to regular SVs to achieve explainability. Depending on the choice of the VF, different kinds of random effects will occur. A detailed discussion of such effects, however, is beyond the scope of this paper.\par
In the present section, we first briefly outline the practical challenge in the evaluation of uncertain SVs and highlight the fundamental difference to the evaluation of regular SVs. Subsequently, we present a simple numerical experiment to demonstrate the effects of randomness of a VF in the context of XML.

\subsection{Evaluation of Uncertain Shapley Values}
The task of calculating regular SVs is NP-hard \cite{deng1994} and requires the evaluation of $2^N$ VFs according to \cref{eqn:phi}. On the other hand, uncertain SVs can by design not be evaluated exactly if no a priori knowledge about the underlying probability distribution is available, \ie, \cref{eqn:deltav:q} is unknown. Instead, they can only be estimated, which requires to evaluate VFs multiple times.\par
To compare the evaluation effort of regular SVs and uncertain SVs, we presume in the following that we explicitly sum over all coalitions in \cref{eqn:phi} and \cref{eqn:phi:E:uncertainty:V}, respectively. For regular Shapey values, this involves $2^N$ VF evaluations. For uncertain SVs, we can repeat each evaluation $n \in \mathbb{N}$ times, which in total requires $n2^N$ VF evaluations. An unbiased estimator for $\tilde{\Phi}_i$ is then given by the sample mean $\overline{\Phi}_i := \sum_{S \subseteq \mathcal{S} \setminus \{i\}} w(S) \left[ \overline{v}(S \cup \{i\}) - \overline{v}(S) \right] $ with $\overline{v}(S) := \frac{1}{n} \sum_{k=1}^n v_k(S)$ based on the \iid VF samples $v_1(S),\dots,v_n(S)$ for all $S \subseteq \mathcal{S}$ and $i \in \mathcal{S}$. As is well-known, a confidence interval of $\overline{\Phi}_i$ for sufficiently many samples reads $\overline{\Phi}_i \pm \propto s_i / \sqrt{n}$ with the corresponding sample variance $s_i^2$.\par
The increased computational complexity in the evaluation of an uncertain SV in comparison with a regular SV ($n2^N$ instead of $2^N$) is the price that has to be paid for an unknown randomness in the VF and its margin depends on the desired confidence of $\overline{\Phi}_i$. A straightforward approximation method for regular SVs is to treat the coalition as a random variable in the sense of \cref{eqn:phi:E} and estimate the expectation value. A similar approach can also be employed with uncertain SVs via \cref{eqn:phi:E:uncertainty:VS}. For an overview over more advanced approximation methods of regular SVs, we refer to \cite{aas2021,touati2021,mitchell2021} and references therein.

\subsection{Numerical Experiment}
In this section, we study our theoretical results in the context of XML. Specifically, we discuss the effects of two different kinds of noises, Bernoulli noise and Gaussian noise, on an exemplary VF (that can in principle be calculated noiselessly) and analyze the resulting uncertain SVs.
To this end, we consider a synthetic data set created from the method \texttt{make\_regression} of the Python library \texttt{scikit-learn} \cite{sklearn2011}.
This method generates a random linear combination of normal distributions.
We sampled $K=\num{10000}$ data points $\mathcal{D}=\lbrace(x^{(i)},y_i)\rbrace_{i=1,\dots,K}$, each point of index $i$ consisting of twelve features $x^{(i)}\in\mathbb{R}^{12}$ and one target value $y_i\in\mathbb{R}$.
All features are informative.
Additionally, we specified an additive Gaussian noise level of $0.1$.
The random seed we used is \num{97531}, which allows for reproducing the data.\par 
Our goal is to study the influence of each feature with respect to the $R^2$ score $R^2(f,\mathcal{D}) := 1- \sum_{i=1}^{K} (y_i-f(x^{(i)}))^2 / \sum_{i=1}^{K} (y_i-\overline{y})^2$ with $\overline{y}:=\frac{1}{K}\sum_i y_i$, which can be used to quantify the quality of a regression model $f$, where $f(x^{(i)})$ describes the prediction of the model for the data point $x^{(i)}$.
The $R^2$ score describes the proportional amount of variation in $y$ that can be predicted from $x$ and takes values in $(-\infty,1]$, with $1$ indicating a perfect match between model and data.
Based on this score, we can define the VF $v(S) = v(S;f,\mathcal{D}) := R^2(f,\{(x_{S|0}^{(1)},y_1),\dots,(x_{S|0}^{(K)},y_K)\})$, where feature values that are not in the coalition $S$ are set to $0$, denoted by $x_{S|0}$ \cite{gromping2007}. As our model of choice, we fit a linear regression with intercept term to the data, \ie, $f_{\theta,b}(x) := \theta^{\top} x +b$.\par
To obtain two use cases, we define two noisy versions of our VF $v(S;f,\mathcal{D})$ by adding random variables that are independent of $S$. First, $\bm{\nu}$ follows a Bernoulli distribution with $p=0.33$ multiplied with the constant $c=0.05$, which has the effect of randomly adding a constant offset to some values. Second, $\bm{\nu}'$ follows a normal distribution with mean $0$ and standard deviation $0.01$. We arrive at three VFs: $v$ (noiseless), $\bm{v}:=v+\bm{\nu}$ (Bernoulli noise), and $\bm{v}':=v+\bm{\nu}'$ (Gaussian noise).\par 
We plot in \cref{fig:pv} the probability mass function $\mathbb{P}_i$ and densities $\tilde{p}_i$ according to \cref{eqn:deltav:p,eqn:deltav:p:uncertainty}, respectively, for all $i\in \mathcal{S} = \lbrace 1,\dots,12\rbrace$. For the Bernoulli noise, the support is discrete and we therefore denote the corresponding probability mass function by $\tilde{\mathbb{P}}_i$. The probability mass function $\mathbb{P}_i$ shown in \cref{fig:pv_noisefree} is the noiseless distribution of $\Delta_iv(S)$ over all coalitions $S$, weighted with their respective $w(S)$ for each feature $i$ as defined in \cref{eqn:deltav:p}.
Its expectation value corresponds to $\Phi_i$ according to \cref{eqn:phi:E}, which is listed in \cref{tab:shapley} for all $i \in \mathcal{S}$.\par
\begin{figure}[t]
	\centering
	\begin{subfigure}[b]{0.32\textwidth}
		\centering
		\includegraphics[width=\textwidth]{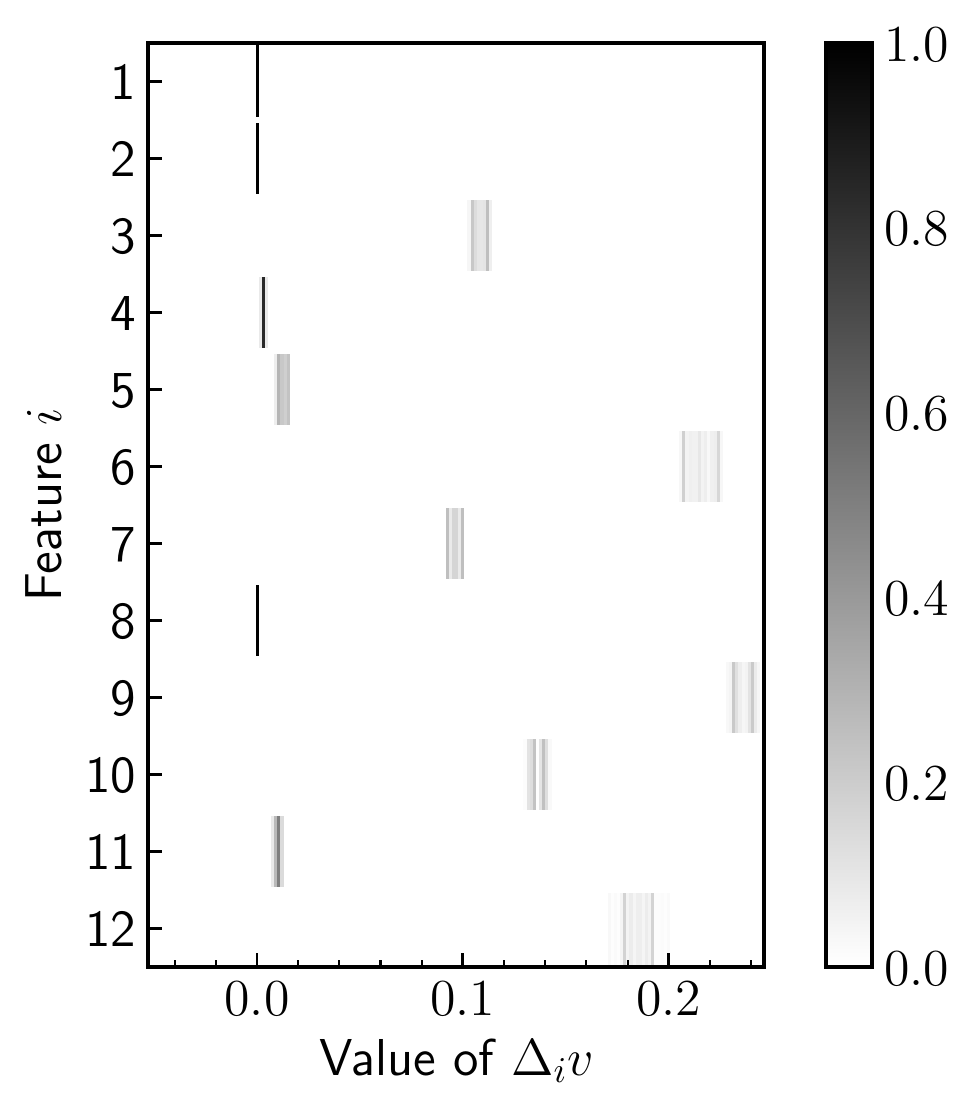}
		\caption{$\BP_i$ for $v$ (noiseless)}
		\label{fig:pv_noisefree}	
	\end{subfigure}%
	\hfill%
	\begin{subfigure}[b]{0.32\textwidth}
		\centering
		\includegraphics[width=\textwidth]{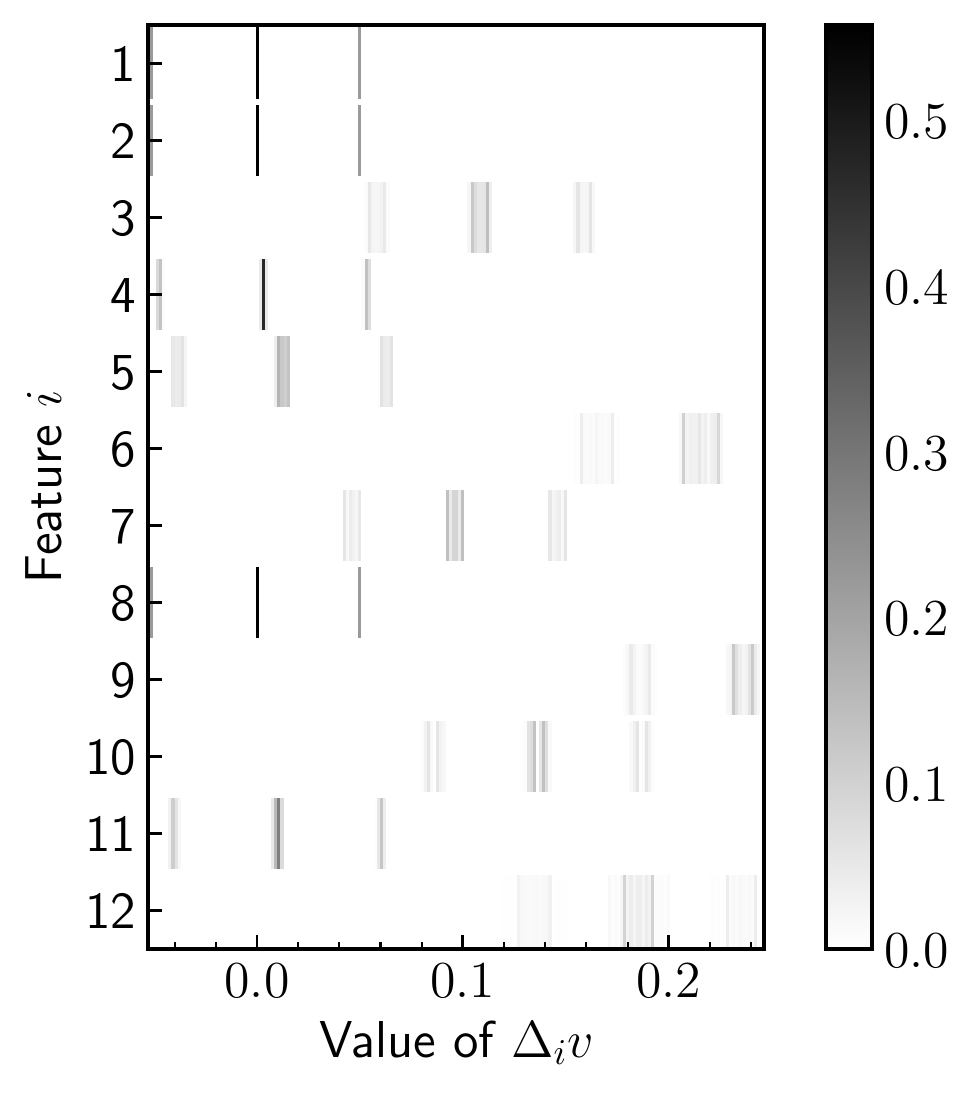}
		\caption{$\tilde{\mathbb{P}}_i$ for $\bm{v}$ (Bernoulli noise)}
		\label{fig:pv_bernoulli}
	\end{subfigure}%
	\hfill%
	\begin{subfigure}[b]{0.32\textwidth}
		\centering
		\vstretch{1.03}{\includegraphics[width=\textwidth]{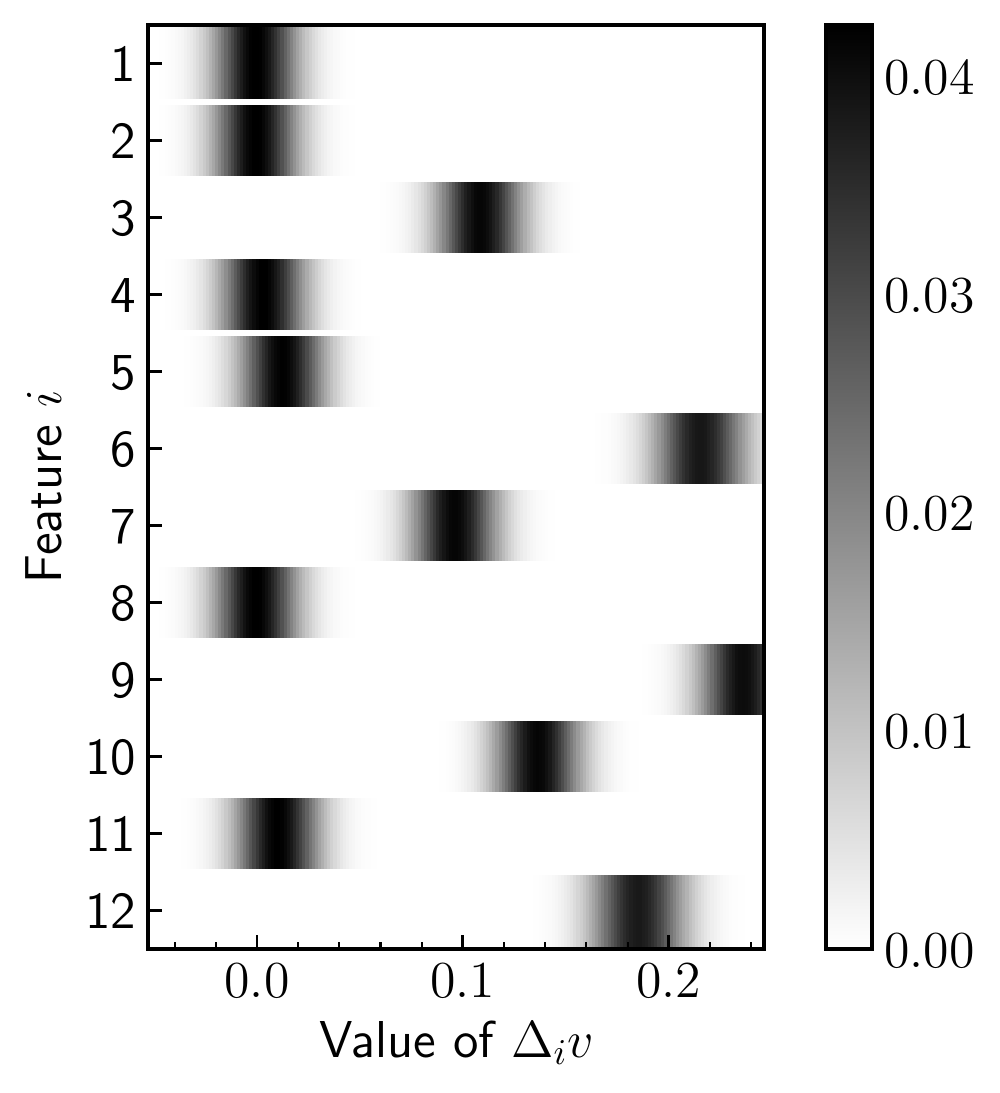}}
		\caption{$\tilde{p}_i$ for $\bm{v}'$ (Gaussian noise)}
		\label{fig:pv_gaussian}
	\end{subfigure}
	\caption{Probability mass functions $\mathbb{P}_i$ and $\tilde{\mathbb{P}}_i$ and probability density $\tilde{p}_i$ according to \cref{eqn:deltav:p,eqn:deltav:p:uncertainty}, respectively, for all features $i\in\lbrace 1,\dots,12\rbrace$.
	Darker colors indicate a higher probability.
	We investigate the effect of additive Bernoulli and Gaussian noise, respectively, on a noiseless VF.}
	\label{fig:pv}
\end{figure}
\begin{table}[t]
	\centering
	\caption{Regular SVs (without uncertainty) and their intrinsic variances \wrt $w(S)$.}
	\setlength{\tabcolsep}{10pt}
	\renewcommand{\arraystretch}{1.15}
	\begin{tabular}{lcc}
		\toprule
		Feature $i$ &  $\Phi_i$ &  $\sigma_i^2$ \\
		\midrule
		$1$  & $-0.000050$ &  $3.081014\times 10^{-9}$ \\
		$2$  & $-0.000009$ &  $1.126325\times 10^{-9}$ \\
		$3$  &  $0.108173$ &  $1.008858\times 10^{-5}$ \\
		$4$  &  $0.003419$ &  $3.832217\times 10^{-7}$ \\
		$5$  &  $0.012733$ &  $3.634072\times 10^{-6}$ \\
		$6$  &  $0.214356$ &  $3.949685\times 10^{-5}$ \\
		$7$  &  $0.095895$ &  $8.100267\times 10^{-6}$ \\
		$8$  &  $0.000009$ &  $2.346858\times 10^{-10}$ \\
		$9$  &  $0.235202$ &  $1.988726\times 10^{-5}$ \\
		$10$ &  $0.135854$ &  $9.707781\times 10^{-6}$ \\
		$11$ &  $0.010408$ &  $1.179264\times 10^{-6}$ \\
		$12$ &  $0.184340$ &  $4.631823\times 10^{-5}$ \\
		\bottomrule
	\end{tabular}
	\label{tab:shapley}
\end{table}
From \cref{eqn:phi:V:uncertainty}, we can derive the variance of all uncertain SVs:
The correlation terms $\xi_i$ are all equal to zero, as $h_i$ is independent of $S$.
The intrinsic variances $\sigma_i^2$ that result from the distribution of the marginal contributions $\Delta_iv$ according to \cref{eqn:phi:V} are listed in \cref{tab:shapley}.
The Bernoulli distribution of $\bm{\nu}$ leads to a symmetrical noise function $h_i$ with mean $0$ and support on $\lbrace -c,0,+c\rbrace$, where both non-zero values occur with probability $p(1-p)$.
The distribution of $\bm{\nu}'$, on the other hand, leads to a Gaussian $h_i$ with twice the original variance.
Thus, the variance of the noise of the marginal contributions is given by $\sigma^2_{\Gamma_i}=c^2\cdot 2p(1-p)=1.1055\times 10^{-3}$ for the Bernoulli noise and $2\cdot 0.01^2=2\times 10^{-4}$ for the Gaussian noise, respectively.
According to \cref{eqn:phi:V:uncertainty}, these values are added to the corresponding $\sigma^2_i$ from \cref{tab:shapley} to obtain the final variance.
Notice that the variance incurred from the Gaussian or Bernoulli noise term is up to 6 orders of magnitude larger than the variances reported in the table.
Hence, the final variance, and consequently the number of samples one has to draw to obtain a specific confidence region is dominated by the noise.
Such insights cannot be derived with existing frameworks with uncertain SVs.

\section{Conclusions} \label{sec:summary}
In this paper, we have proposed uncertain SVs as a generalization of regular SVs with VFs that are represented by random variables instead of deterministic functions.
With this approach, we can consider uncertainties in the evaluation of VFs, \eg, due to a non-deterministic behavior of ML algorithms.
Based on our definition in form of expectation values, we have found that uncertain SVs correspond to regular SVs with shifted VFs, where the shift is determined by the mean bias of the marginal contributions.
If no such uncertainty is present (\eg, because the noise of the VF is independent of the corresponding coalition), both kinds of SVs coincide.\par 
The practical evaluation of uncertain SVs can (without a priori knowledge about the uncertainty) only be realized by a sample mean that requires repeated VF evaluations.
Thus, a key difference between uncertain SVs and regular SVs is the higher computational effort required to achieve a desired level of confidence due to the VF uncertainty.\par
We consider our work as a solid mathematical framework based on first principles that is highly general and can be used as a potential starting point for further research.
For example, with respect to different kinds of noises, \eg, multiplicative noise or noise that is correlated for different coalitions.
Another possible research direction is the investigation of uncertainty that arises from typical VFs in an XML context.
Furthermore, we think that a study of the effects of uncertain VFs on specialized SV approximations such as \emph{KernelSHAP} \cite{lundberg2017}, \emph{TreeSHAP} \cite{lundberg2020} or \emph{DeepSHAP} is potentially insightful.
It also remains an open question how our uncertainty treatment can be translated to related concepts like Owen values \cite{lopez2009} or Banzhaf-Owen values \cite{saavedranieves2021}.
Finally, quantum machine learning \cite{cerezo2022} is a particularly promising field of application beyond XML because randomness is an inherent property of quantum systems.

\subsubsection{Acknowledgments}
The authors would like to thank Sabine Müller and Moritz Wolter for helpful discussions and constructive feedback.
Parts of this research have been funded by the Federal Ministry of Education and Research of Germany and the state of North-Rhine Westphalia as part of the Lamarr-Institute for Machine Learning and Artificial Intelligence (LAMARR22B), as well as by the Fraunhofer Cluster of Excellence Cognitive Internet Technologies (CCIT) and by the Fraunhofer Research Center Machine Learning (FZML).

\bibliography{literature} 

\end{document}